\documentclass{article}

\usepackage[main, final]{neurips_2026}
\usepackage[utf8]{inputenc}
\usepackage[T1]{fontenc}
\usepackage{hyperref}
\usepackage{url}
\usepackage{booktabs}
\usepackage{amsfonts}
\usepackage{microtype}
\usepackage{xcolor}

\usepackage{amsmath}
\usepackage{amssymb}
\usepackage{amsthm}

\usepackage{graphicx}
\usepackage{subcaption}
\usepackage{caption}
\usepackage{multirow}
\usepackage{enumitem}

\newtheorem{mytheorem}{Theorem}
\newtheorem{mylemma}{Lemma}
\newtheorem{myproposition}{Proposition}
\newtheorem{mycorollary}{Corollary}

\DeclareMathOperator*{\argmax}{argmax}
\title{Geometric Inductive Biases for Semi-Supervised Equalization: The Constellation-Aware Transformer}

\author{
  Avi Caciularu \\[5pt]
  \normalfont Google Research
}

\begin{document}

\maketitle

\begin{abstract}
Decoding signals over unknown channels with minimal pilot overhead is a critical challenge in next-generation communications. Existing deep learning approaches typically rely on generic encoders that struggle to model long-range temporal dependencies or efficiently capture the channel's physical properties from scarce data. We argue that standard architectures suffer from \textit{agnostic estimation gaps}, as they must implicitly learn the constellation geometry that is already known. We introduce the \textit{Constellation-Aware Transformer} (CAT), a novel architecture that explicitly injects geometric inductive biases into the equalization process. CAT is composed of a stack of custom \textit{TransFIRmer} blocks, which use an ``early interaction'' paradigm to co-process received signals and ideal constellation symbols. Each block features a split Feed-Forward Network that applies a Finite Impulse Response (FIR)-inspired filter for deconvolution and a parallel MLP for geometric refinement. We show that this design is structurally aligned with the optimal linear (MIMO Wiener) receiver: its attention can implement a matched-filter bank, and its bidirectional FIR branch provides the non-causal filtering that block MMSE equalization requires. In the semi-supervised setting, CAT needs fewer pilots than VAE and standard Transformer baselines: on two of our three ISI channels, it reaches a lower SER with 64 pilots than they do with 128.
\end{abstract}

\section{Introduction}

Deep learning methods for communications over unknown channels have attracted considerable interest, promising to replace rigid, model-based algorithms with flexible, data-driven solutions \citep{8054694,8437530,8242643,8267032,shlezinger2020viterbinet,9252949}. A central challenge is to minimize the pilot data required for reliable decoding, as this overhead directly reduces spectral efficiency \citep{9242305}.

While classical methods like Expectation-Maximization (EM) offer a path to unsupervised equalization \citep{720247,em}, their performance is limited by initialization sensitivity and local minima. Deep generative models, specifically Variational Autoencoders (VAEs) \citep{vae}, have been adapted to a semi-supervised learning (SSL) framework \citep{kingma2014semi}, demonstrating significant pilot reduction \citep{caciularu1,caciularu2,burshtein1,burshtein2}. However, the encoders in these systems, typically MLPs or CNNs, may not fully capture complex temporal dependencies, particularly under severe Inter-Symbol Interference (ISI).

The Transformer architecture \citep{NIPS2017_3f5ee243} motivates application to this domain, yet a direct adaptation is often suboptimal \citep{choukroun2024a}. A generic Transformer has to learn the underlying physics, such as the constellation geometry and the channel's filtering, entirely from data, even though communication theory already provides much of it.

We instead build this knowledge into the Transformer's internal components. In NLP, processing a query and a document jointly from the first layer (\textit{early and deep interaction}) is more accurate than encoding them separately \citep{Humeau2019PolyencodersAAA,caciularu-etal-2021-cdlm-cross,fang-etal-2020-hierarchical}. We do the same with the received signal and the ideal constellation symbols.

We introduce the \textit{Constellation-Aware Transformer (CAT)}, composed of novel \textit{TransFIRmer} blocks (Figure~\ref{fig:cat}). Each block changes the standard Transformer encoder in two ways:
\begin{enumerate}[leftmargin=1.5em]
    \item \textbf{Constellation-Aware Attention:} Co-processes received signals and ideal constellation symbols. The model is given the exact constellation, and its attention can act as a matched-filter bank.
    \item \textbf{TransFIRmer FFN:} Replaces the standard MLP with two streams: a bidirectional FIR-inspired filter on the signal tokens for deconvolution, and a parallel MLP on the constellation tokens.
\end{enumerate}

The analysis in Section~\ref{sec:theory} motivates these two choices. It shows that an encoder given the constellation can do no worse than one without it, that CAT's attention can act as a matched-filter bank, and that its bidirectional FIR branch supplies the non-causal filtering needed for block MMSE equalization. These are statements about capacity, not about what training finds. In our experiments, CAT outperforms VAE and standard Transformer baselines on a nonlinear memoryless channel and, once at least 32 pilots are available, on three synthetic ISI channels. The advantage holds on the 3GPP TDL-A, TDL-C and TDL-D profiles and carries over to coded transmission with a 5G NR LDPC code.

\begin{figure}[t]
    \centering
    \includegraphics[width=\linewidth]{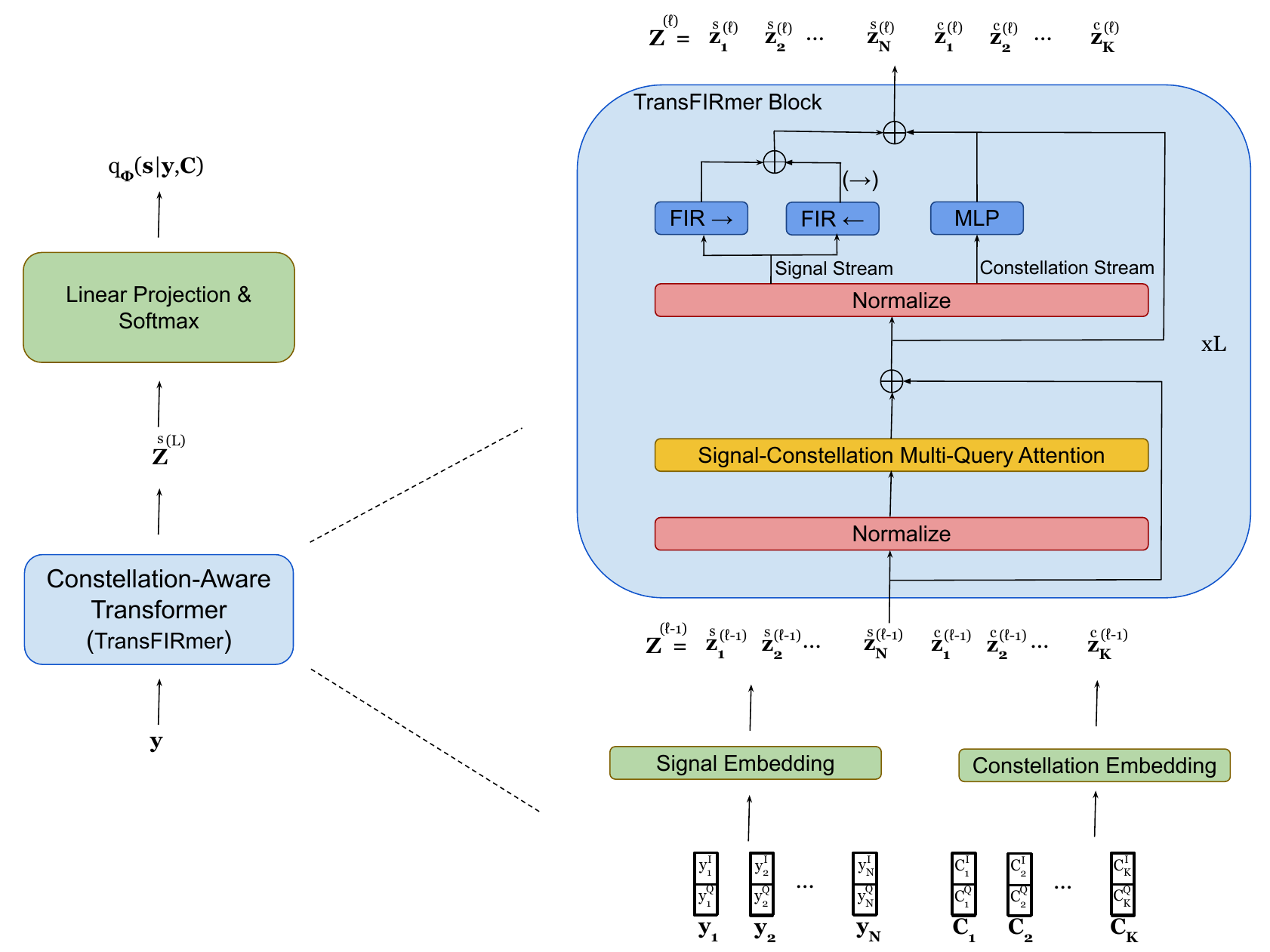}
    \caption{The architecture of our Constellation-Aware Transformer (CAT). The model processes the received signals $\mathbf{y}$ and ideal constellation symbols $\mathcal{C}$ through a stack of TransFIRmer blocks, enabling deep interaction between signal and geometry.}
    \label{fig:cat}
\end{figure}

\section{Problem Formulation and Setup}
\label{sec:problem_setup}

We consider a block of $N$ symbols, $(s_1, \ldots, s_N)$, transmitted over an unknown channel. Each symbol $s_i \in \{1, \dots, K\}$ is drawn independently and uniformly and mapped to a signal from a fixed constellation $\mathcal{C}$ of size $K$. We represent these signals as real-valued vectors $\mathbf{x}_i = (x_i^I, x_i^Q) \in \mathbb{R}^2$. The first $N_p \ll N$ symbols are known pilots and the rest form the unknown payload.

\subsection{Channel Models}

\subsubsection{Memoryless Channels}
In a memoryless channel, the received signal depends only on the current transmitted signal. We model the transformation as an unknown function $h(\cdot)$, typically including I/Q imbalance and fading. The received signal is $\mathbf{y}_i = h(\mathbf{x}_i) + \mathbf{n}_i$, where $\mathbf{y}_i = (y_i^I, y_i^Q) \in \mathbb{R}^2$.

\subsubsection{Channels with Finite Memory}
In channels with finite memory, the received signal $\mathbf{y}_i$ is affected by ISI. Following \citet{burshtein2}, we define this as a memoryless nonlinearity $g(\cdot)$ followed by a linear channel impulse response:
\begin{equation}
    \mathbf{y}_i = \sum_{l=0}^{L-1} \mathbf{h}_l g(\mathbf{x}_{i-l}) + \mathbf{n}_i,
    \label{eq:isi_model}
\end{equation}
where the complex-valued filter taps $\{\mathbf{h}_l\}_{l=0}^{L-1}$ and the nonlinearity $g(\cdot)$ are unknown to the receiver.

\subsection{The Semi-Supervised Variational Framework} \label{sec:ssl_framework}
To reduce the dependency on pilots, we adopt a semi-supervised learning (SSL) framework. The VAE-based approach \citep{caciularu1,caciularu2,burshtein1,burshtein2} involves an \textbf{encoder} $q_{\mathbf{\phi}}(s|\mathbf{y})$, approximating $p(s|\mathbf{y})$, and a \textbf{decoder} $p_{\mathbf{\theta}}(\mathbf{y}|s)$, modeling the forward channel. They are trained jointly by minimizing:
\begin{equation}
\begin{split}
\mathcal{L}_{\text{SSL}} &=
-\frac{\alpha}{N_{p}} \sum_{i=1}^{N_{p}} \log q_{\mathbf{\phi}}(s_i|\mathbf{y}_i)
- \frac{\gamma}{N_{p}} \sum_{i=1}^{N_{p}} \log p_{\mathbf{\theta}}(\mathbf{y}_i|s_i) \quad + \\ &\frac{1-\gamma}{N-N_{p}} \sum_{i=N_{p}+1}^{N} \bigg[-\mathbb{E}_{q_{\mathbf{\phi}}(s|\mathbf{y}_i )}[\log p_{\mathbf{\theta}}(\mathbf{y}_i|s)] + D_{KL}(q_{\mathbf{\phi}}(s|\mathbf{y}_i ) \| p(s)) \bigg],
\end{split}
\label{eq:ssl_vae}
\end{equation}
where $\alpha$ and $\gamma$ are weighting hyperparameters. CAT replaces the encoder $q_{\mathbf{\phi}}(s|\mathbf{y})$.\footnote{For ISI channels, the encoder conditions on the entire sequence $\mathbf{y}$, i.e., $q_{\mathbf{\phi}}(s_i|\mathbf{y})$. On the memoryless channel, CAT also attends over the whole block, because the fading coefficient and the I/Q imbalance are the same for all its symbols.}

\section{Proposed Method: The Constellation-Aware Transformer}
\label{sec:method}

We first motivate the design of CAT theoretically and then describe its architecture.

\subsection{Theoretical Analysis}
\label{sec:theory}

Standard equalizers approximate the inverse channel function purely from data. We first explain why ignoring the known constellation is a disadvantage, and then show that our architecture is structurally aligned with the optimal linear receiver. These results concern representational capacity and population risk; whether training realizes them is tested in Section~\ref{sec:experiments}.

\subsubsection{The Cost of Agnosticism: The Estimator Gap}
Let $s$ be a symbol from constellation $\mathcal{C}$, and let $\hat{s}^*(\mathbf{y}) = \mathbb{E}[s | \mathbf{y}, \mathcal{C}^*]$ be the optimal MMSE estimator conditioned on the true constellation $\mathcal{C}^*$. A domain-agnostic model must implicitly learn a posterior approximation $q_\phi(\mathcal{C}|\mathbf{y})$ from scarce pilot data. We quantify the penalty via the \textit{Excess Bayesian Risk}. In practice the receiver knows the constellation. The lemma concerns encoders that do not take it as input. Such an encoder has to represent the geometry implicitly in its weights, and $q_\phi(\mathcal{C}|\mathbf{y})$ is an idealized description of the uncertainty that remains, not a component of any baseline.

\begin{mylemma}[The Agnostic Estimator Gap]
\label{lemma:estimator_gap}
Assume the conditional expectation function $m(\mathcal{C}) = \mathbb{E}[s | \mathbf{y}, \mathcal{C}]$ is bounded by a constant $R$. The expected excess MSE of a constellation-agnostic estimator $\hat{s}_\phi$ relative to the optimal estimator $\hat{s}^*$ is upper-bounded by:
\begin{equation}
\label{eq:estimator_gap}
    \mathcal{R}(\hat{s}_\phi) - \mathcal{R}(\hat{s}^*) \leq 2R^2 \mathbb{E}_{\mathbf{y}} \left[ D_{KL}( \delta_{\mathcal{C}^*} || q_\phi(\mathcal{C}|\mathbf{y}) ) \right].
\end{equation}
\end{mylemma}
\begin{proof}[Proof Sketch]
We bound $\|\hat{s}_\phi - \hat{s}^*\|$ via the total variation distance between $q_\phi$ and $\delta_{\mathcal{C}^*}$, then apply Pinsker's inequality. Full derivation in Appendix~\ref{app:proofs}.
\end{proof}

\begin{mycorollary}[Finite-Sample Pilot Efficiency, informal]
\label{cor:pilot_efficiency}
Under PAC-Bayes assumptions, for a model with $M$ parameters trained on $N_p$ pilots, $\mathbb{E}[D_{KL}(\delta_{\mathcal{C}^*} || q_\phi)] = \mathcal{O}(M / N_p)$. Combined with Lemma~\ref{lemma:estimator_gap}, the excess risk of a constellation-agnostic estimator decays as $\mathcal{O}(M / N_p)$.
\end{mycorollary}

\noindent CAT receives $\mathcal{C}^*$ as input, so in this idealized picture its gap is zero. CAT still has to learn the channel from the pilots; what it no longer has to learn is the constellation.

\subsubsection{Hypothesis Space Superiority}
We further formalize this advantage via the available hypothesis spaces.
\begin{mytheorem}[Advantage of Parametric Hypothesis Space]
\label{thm:hypothesis_space}
Let $\mathcal{G}_{\text{std}}$ be the space of functions mapping $\mathbf{y} \to \hat{s}$, and $\mathcal{G}_{\text{CA}}$ be the space of functions mapping $(\mathbf{y}, \mathcal{C}) \to \hat{s}$. The achievable MMSE satisfies $\text{MMSE}_{\text{CA}} \leq \text{MMSE}_{\text{std}}$.
\end{mytheorem}
\begin{proof}[Proof Sketch]
$\mathcal{G}_{\text{std}} \subseteq \mathcal{G}_{\text{CA}}$ since any $g \in \mathcal{G}_{\text{std}}$ extends to $\psi(\mathbf{y}, \mathcal{C}) = g(\mathbf{y})$. The infimum over the superset cannot exceed that over the subset. On its own, this result says nothing about finite samples; Appendix~\ref{app:proofs} gives a heuristic argument for why not having to learn the constellation should reduce the number of pilots needed.
\end{proof}

\subsubsection{Structural Alignment with Optimal Filtering}
Classical theory states that the optimal linear receiver for ISI channels consists of a \textit{Block MMSE Equalizer} followed by a \textit{Matched Filter}.\footnote{The mathematical structure is equivalent to the MIMO Wiener filter when ISI is modeled as virtual spatial multiplexing.}

\begin{myproposition}[Structural Alignment with Optimal Receiver]
\label{prop:alignment}
The CAT architecture provides inductive biases aligned with optimal receiver components:
\begin{enumerate}
    \item \textbf{Matched Filter:} The Constellation-Aware Attention has the capacity to implement a Matched Filter bank through learned projections.
    \item \textbf{Block MMSE Equalizer:} The TransFIRmer's bidirectional FIR-FFN provides the non-causal spectral bias conducive to deconvolution.
\end{enumerate}
\end{myproposition}
\begin{proof}[Proof Sketch]
For (1), expanding the Gaussian posterior shows detection reduces to cross-correlations $\mathbf{z}^T c_k$, which the attention mechanism computes. For (2), the optimal equalizer $\mathbf{W} = \mathbf{H}^H \mathbf{R}_{yy}^{-1}$ requires a non-causal FIR filter, provided by our bidirectional convolutions. Full proof and a discussion of non-constant modulus constellations in Appendix~\ref{app:proofs}.
\end{proof}

\noindent Proposition~\ref{prop:alignment} is about capacity; we do not claim that training converges to the Wiener solution. On $h^{(1)}$ and $h^{(2)}$, CAT's gap to the optimal decoder shrinks as $N_p$ grows (Table~\ref{tab:isi_results}).

\subsection{The CAT Architecture}
CAT is a stack of 3 layers, which we call \textit{TransFIRmer blocks}.

\subsubsection{Input Representation and Embedding}
The input consists of the sequence of $N$ received channel outputs $(\mathbf{y}_{1}, \ldots, \mathbf{y}_{N})$ and the set of $K$ ideal constellation symbols $\{\mathbf{x}(1), \ldots, \mathbf{x}(K)\}$. Both are projected into dimension $d_{\text{model}}$. We add fixed sinusoidal positional embeddings $\mathbf{p}_i$ to the signals, but not to the set-based constellation symbols:
\begin{align*}
    \mathbf{z}^{\text{sig}}_i = \text{Embed}_{\text{sig}}(\mathbf{y}_i) + \mathbf{p}_i, \quad i \in \{1, \ldots, N\},\\
    \mathbf{z}^{\text{const}}_k = \text{Embed}_{\text{const}}(\mathbf{x}(k)), \quad k \in \{1, \ldots, K\}.
\end{align*}

\subsubsection{The TransFIRmer Block}
The TransFIRmer block modifies both Transformer sub-layers to align with Proposition~\ref{prop:alignment}.

\paragraph{Signal-Constellation Attention Mechanism.}
This stage enables deep interaction via joint self-attention over the concatenated $N+K$ tokens ($\mathbf{Q}$, $\mathbf{K}$, $\mathbf{V}$ computed from the full sequence via multi-query attention). We use a full (bidirectional) attention mask, since equalization is a \textit{smoothing} problem.

\paragraph{Two-Stream Feed-Forward Network.}
The second stage replaces the MLP with two streams.
\begin{itemize}[leftmargin=1.5em]
    \item \textbf{Signal Stream (Bidirectional FIR-Inspired Filter):} The signal tokens $\mathbf{Z}^{(l)}_{\text{sig}}$ are processed by a pair of 1D convolutional layers implementing a bidirectional filter:
    \begin{equation}
    \begin{split}
        \text{FFN}_{\text{sig}}(\mathbf{Z}_{\text{sig}}) = \text{Conv}_{\text{fwd}}(\mathbf{Z}_{\text{sig}})
         + \text{Flip}\left(\text{Conv}_{\text{bwd}}\left(\text{Flip}\left(\mathbf{Z}_{\text{sig}}\right)\right)\right).
    \end{split}
    \label{eq:firffn}
    \end{equation}
    This sub-layer can represent the channel-matched filter $\mathbf{H}^H$ of the optimal MMSE equalizer (Proposition~\ref{prop:alignment}). Since the signals are in $\mathbb{R}^2$, it can also learn the complex conjugation in $\mathbf{H}^H$.

    \item \textbf{Constellation Stream (MLP):} A standard MLP refines the constellation tokens $\mathbf{Z}^{(l)}_{\text{const}}$.
\end{itemize}
Each stream has its own residual connection.

\subsection{Symmetry Breaking and Inductive Bias}
A limitation of standard Transformers is their permutation equivariance: they lack an intrinsic ``anchor'' for the signal space. In 16-QAM, the difference between symbols is geometric and absolute, not relative. CAT breaks this symmetry by feeding fixed constellation tokens $\mathcal{C}$ into attention, providing a rigid coordinate system against which received signals are measured. The attention scores $A_{ij} \propto \mathbf{y}_i^T \mathbf{c}_j$ perform soft-decision mapping to symbol centroids, transforming the problem from ``learning clustering from scratch'' to ``registering points to a known template''.

\subsection{Computational Complexity Analysis}
\label{sec:complexity}
The complexity of standard self-attention over $N$ tokens is $\mathcal{O}(N^2 d)$. In CAT, we process $N+K$ tokens, giving $\mathcal{O}((N+K)^2 d)$. The attention cost therefore grows by a factor of $(1+K/N)^2$. For 16-QAM ($K=16$), the overhead is about 56\% at $N=64$, 13\% at $N=256$ and 6\% at $N=512$. The pilot savings reduce transmission overhead, not this extra computation at the receiver.

\paragraph{Scalability to High-Order Modulation.} Table~\ref{tab:scalability} shows how CAT scales with the constellation order on $h^{(1)}$. Its SER gain over the vanilla Transformer shrinks from 3.8\,dB for 16-QAM to 3.0\,dB for 64-QAM and 2.2\,dB for 256-QAM, while the attention overhead grows to about 56\% for 64-QAM and 300\% for 256-QAM. For 64-QAM the pilot savings translate into a 15\% gain in spectral efficiency. Appendix~\ref{app:scalability} discusses higher orders and compares parameter counts.

\begin{table}[t]
\centering
\caption{Scalability across constellation orders ($h^{(1)}$ channel, $N\!=\!256$, SNR 16--24~dB, gains at target SER $10^{-2}$).}
\label{tab:scalability}
\begin{tabular}{@{}lccc@{}}
\toprule
\textbf{Modulation} & \textbf{$K$} & \textbf{Attn. Overhead} & \textbf{SER Gain vs. Vanilla Trans.} \\
\midrule
16-QAM & 16 & $\sim$13\% & 3.8 dB \\
64-QAM & 64 & $\sim$56\% & 3.0 dB \\
256-QAM & 256 & $\sim$300\% & 2.2 dB \\
\bottomrule
\end{tabular}
\end{table}

\section{Experiments}
\label{sec:experiments}
We evaluate CAT as a semi-supervised channel equalizer. The experiments ask whether its two inductive biases, the constellation prior with early interaction and the TransFIRmer block, allow sample-efficient adaptation to channel impairments without large offline training sets. We report results on a nonlinear memoryless channel and on three ISI channels, followed by an ablation. Section~\ref{sec:discussion} adds results with channel coding and on 3GPP TDL channels.

\subsection{Experimental Setup}
\label{sec:experiments_setup}

To evaluate CAT, we test it on two categories of channels: a nonlinear memoryless channel with I/Q imbalance and Rayleigh fading, and three standard channels with finite memory (ISI). Unless stated otherwise, experiments use 16-QAM.

Our CAT model consists of a 3-layer stack of TransFIRmer blocks with Multi-Query Attention \citep{shazeer2019fast} and fixed sinusoidal positional embeddings. CAT achieves its performance with approximately $10$k parameters, comparable to the vanilla Transformer ($9$k) and far fewer than VAE-CNN ($150$k; see Appendix~\ref{app:scalability}). Other hyperparameters are consistent with \cite{burshtein2} (see Appendix~\ref{app:hyperparams}).

We compare against the optimal decoder (ML/BCJR) \citep{bcjr}, SSL Monte Carlo EM (MCEM) \citep{MCEM}, SSL Viterbi EM \citep{em}, Simple Decision Directed (SDD), VAE-CNN \citep{burshtein2}, and the meta-learning algorithm CAVIA \citep{cavia}.\footnote{CAVIA requires meta-training on prior channel realizations, giving it access to distributional knowledge unavailable to other methods under strict test-time training constraints. We include it as an oracle-assisted baseline.} We also include a vanilla Transformer baseline \citep{kunde2025transformers} to isolate the benefits of our architectural modifications.
We report the mean Symbol Error Rate (SER) over 1000 independent Monte Carlo trials; the 95\% confidence intervals are negligible.

\paragraph{SNR Range and Training Protocol.} Most experiments use moderate-to-high SNR (17--22\,dB): this range is where the semi-supervised challenge is most relevant, it matches practical systems (e.g., 5G NR targets 15--25\,dB), and it allows direct comparison with prior work \citep{burshtein1,burshtein2}. Table~\ref{tab:snr_sweep} extends the range to 14--26\,dB on $h^{(1)}$. Each block experiences a single channel realization. CAT, the vanilla Transformer and VAE-CNN are trained from scratch on every block for 5,000 steps with Eq.~(\ref{eq:ssl_vae}), using only that block's pilots and its unlabeled payload, and SER is measured on the same payload. No payload labels and no data from other channel realizations are used, except by the CAVIA baseline and in the meta-initialized runs of Section~\ref{sec:discussion}. The model is therefore never tested on a channel other than the one it has just adapted to; when the channel changes, it adapts again. Residual carrier frequency offset (CFO) and phase noise within a block are covered in Appendix~\ref{app:cfo}, and other channel types by the 3GPP TDL profiles (Section~\ref{sec:3gpp_eval}).

\subsection{Results on Memoryless Channels}
Figure~\ref{fig:main_results_memoryless} shows the SER as a function of the number of payload symbols at 18\,dB and 22\,dB. CAT has a lower SER than the VAE-CNN at every payload size, and the gap is largest for the shortest payload (64 symbols). As payload size increases, CAT closely approaches the Optimal decoder, suggesting that the ``Early Interaction'' prior effectively substitutes for large training sets.

\begin{figure}[t]
    \centering
    \begin{subfigure}[t]{0.49\linewidth}
        \includegraphics[width=\textwidth]{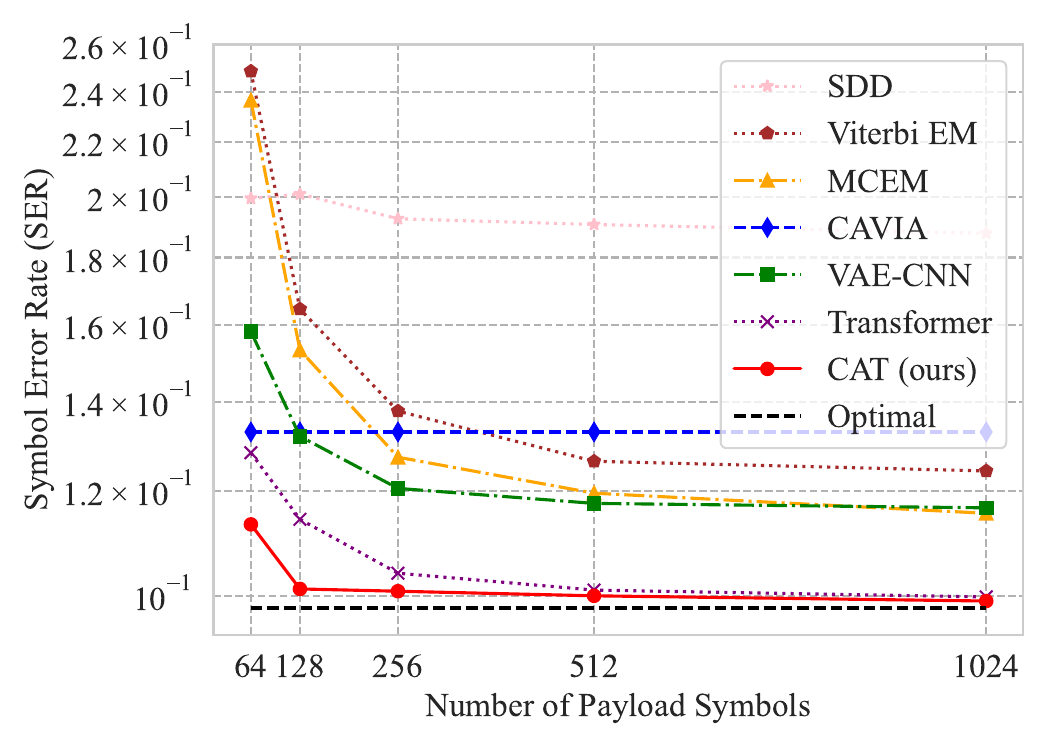}
        \caption{SER results for SNR $=18$\,dB.}
        \label{fig:ser_vs_n_18db}
    \end{subfigure}
    \hfill
    \begin{subfigure}[t]{0.48\linewidth}
        \includegraphics[width=\textwidth]{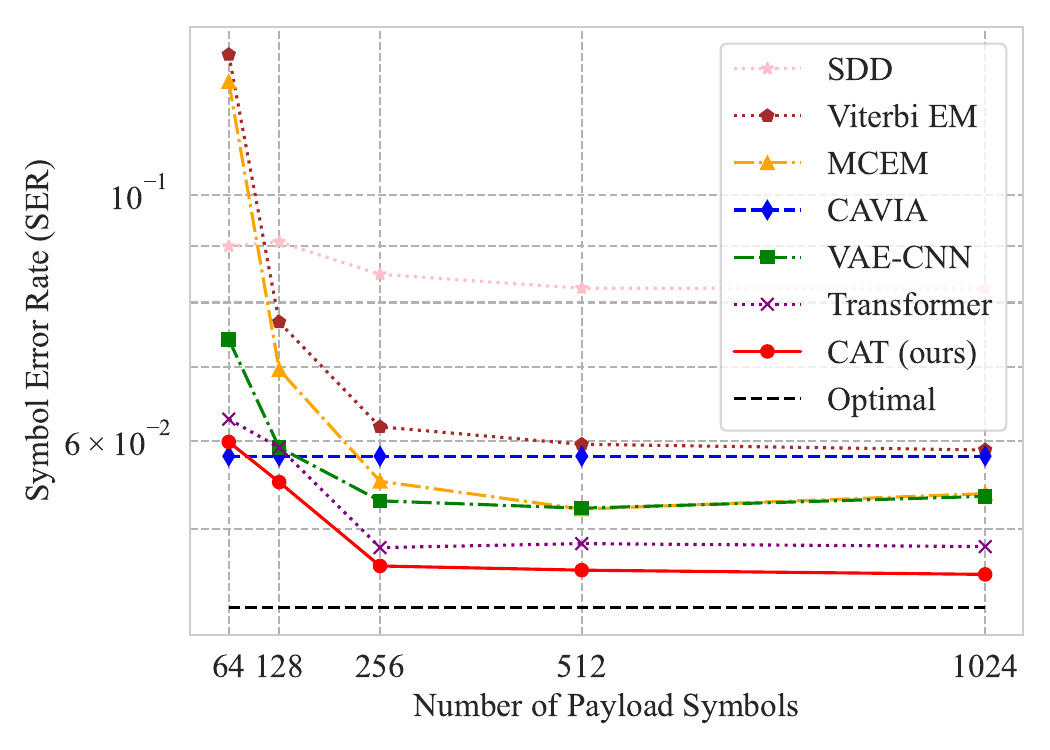}
        \caption{SER results for SNR $=22$\,dB.}
        \label{fig:ser_vs_n_22db}
    \end{subfigure}
    \caption{Symbol Error Rate (SER) on a memoryless nonlinear channel as a function of the number of payload symbols, with $N_p=16$ pilots. Apart from the optimal decoder, only the meta-trained CAVIA baseline has a lower SER than CAT, and only at 64 payload symbols and 22\,dB.}
    \label{fig:main_results_memoryless}
\end{figure}

\subsection{Results on Channels with Memory (ISI)}
We evaluate CAT on three standard ISI channels (see Appendix~\ref{sec:channs_with_mem}) with different lengths ($L\!=\!5, 4, 10$). Table~\ref{tab:isi_results} summarizes the results. With $N_p=16$, one pilot per constellation point on average, all three learned equalizers do poorly (SER $0.29$--$0.65$), and on $h^{(1)}$ CAT trails both baselines. From $N_p=32$ on, CAT is the best of the three on every channel. For $N_p\ge 64$ it reduces SER relative to the stronger baseline by $1.9$--$3.4\times$ on $h^{(1)}$ and $h^{(2)}$ and by $1.2$--$1.4\times$ on the 10-tap $h^{(3)}$, where BCJR is intractable. On $h^{(1)}$ and $h^{(2)}$, CAT with 64 pilots also has a lower SER than either baseline with 128, and with 128 pilots it comes within a factor of $1.3$ and $2.5$ of the BCJR optimum, compared with $2.4$ and $8.6$ for the stronger baseline. The ablation in Table~\ref{tab:ablation} attributes most of the gain on $h^{(1)}$ to the bidirectional FIR-FFN and the constellation prior. Table~\ref{tab:snr_sweep} repeats the $h^{(1)}$ comparison from 14 to 26\,dB: CAT has the lowest SER at every SNR, and its margin over the stronger baseline grows from $1.5\times$ at 14\,dB to $3.9\times$ at 26\,dB.

\begin{table}[t]
\centering
\caption{SER on channels with memory ($E_x/N_0 = 17$\,dB, payload 256). CAT vs.\ VAE-CNN and Vanilla Transformer.}
\label{tab:isi_results}
\begin{tabular}{@{}llcccc@{}}
\toprule
\textbf{Channel} & \textbf{Pilots} & \textbf{VAE-CNN} & \textbf{Trans.} & \textbf{CAT (Ours)} & \textbf{Optimal} \\
\midrule
\multirow{4}{*}{$h^{(1)}$ (L=5)}
& 16 & \textbf{0.2900} & 0.3392 & 0.3580 & \multirow{4}{*}{0.0121} \\
& 32 & 0.1251 & 0.1192 & \textbf{0.0842} & \\
& 64 & 0.0523 & 0.0610 & \textbf{0.0198} & \\
& 128 & 0.0494 & 0.0290 & \textbf{0.0156} & \\
\cmidrule(l){2-6}
\multirow{4}{*}{$h^{(2)}$ (L=4)}
& 16 & 0.3447 & 0.3563 & \textbf{0.3330} & \multirow{4}{*}{0.0101} \\
& 32 & 0.1843 & 0.1593 & \textbf{0.1372} & \\
& 64 & 0.1002 & 0.0750 & \textbf{0.0340} & \\
& 128 & 0.0869 & 0.0888 & \textbf{0.0257} & \\
\cmidrule(l){2-6}
\multirow{4}{*}{$h^{(3)}$ (L=10)}
& 16 & 0.6211 & 0.6481 & \textbf{0.6103} & \multirow{4}{*}{N/A} \\
& 32 & 0.3943 & 0.3874 & \textbf{0.3426} & \\
& 64 & 0.1709 & 0.1692 & \textbf{0.1181} & \\
& 128 & 0.1087 & 0.1022 & \textbf{0.0846} & \\
\bottomrule
\end{tabular}
\end{table}
\begin{table}[t]
\centering
\caption{SER vs.\ SNR on $h^{(1)}$ ($N_p=64$, payload 256). The 17\,dB column repeats Table~\ref{tab:isi_results}.}
\label{tab:snr_sweep}
\begin{tabular}{@{}lccccc@{}}
\toprule
\textbf{Method} & \textbf{14\,dB} & \textbf{17\,dB} & \textbf{20\,dB} & \textbf{24\,dB} & \textbf{26\,dB} \\
\midrule
VAE-CNN        & 0.1842 & 0.0523 & 0.0215 & 0.0068 & 0.0031 \\
Vanilla Trans. & 0.1763 & 0.0610 & 0.0248 & 0.0091 & 0.0042 \\
CAT (Ours)     & \textbf{0.1205} & \textbf{0.0198} & \textbf{0.0072} & \textbf{0.0019} & \textbf{0.0008} \\
\bottomrule
\end{tabular}
\end{table}
\subsection{Ablation Study}
Table~\ref{tab:ablation} reports ablations on the memoryless channel and, for the main components, on $h^{(1)}$.
\begin{enumerate}[leftmargin=1.5em]
    \item \textbf{Importance of FIR-FFN:} Replacing the FIR filter with a standard MLP (``CAT with MLP-FFN'') changes little on the memoryless channel ($0.0599$ vs.\ $0.0615$), where there is no ISI to undo. On $h^{(1)}$ the same change raises SER from $0.0198$ to $0.0351$ ($+77\%$), and removing only the backward branch raises it to $0.0287$ ($+45\%$), as expected for non-causal equalization. On $h^{(1)}$ and $h^{(2)}$, the FIR gain is positive at every SNR we tested (7--77\% on $h^{(1)}$ and 30--71\% on $h^{(2)}$) and smallest at 14\,dB, where noise rather than residual ISI limits performance (Appendix~\ref{app:fir_breadth}).
    \item \textbf{Constellation Prior:} With the FIR-FFN kept, removing the prior raises the ISI SER from $0.0198$ to $0.0438$ ($2.2\times$). Two variants receive the same input as CAT: ``CAT with MLP-FFN'' has the constellation tokens and joint attention but a standard FFN, and ``Self-Only Attention'' has the tokens but no signal--constellation attention. Both are worse than CAT ($0.0615$ and $0.0621$ vs.\ $0.0599$), so having the constellation in the input is not enough by itself. A $45^\circ$ \textit{rotated} prior is much worse than no prior at all ($0.1550$ vs.\ $0.0628$), which indicates that the model does use the constellation as a geometric reference. This also means that residual CFO and phase noise, which rotate the received constellation, could hurt CAT. With coarse PLL tracking and a learned phase rotation of the constellation keys (Appendix~\ref{app:cfo}), CAT keeps a 1.5\,dB advantage over the baselines under severe Wiener phase noise ($\sigma^2_\Delta = 10^{-3}$\,rad$^2$/symbol), with no error floor. Adding positional embeddings to the constellation tokens, which have no natural order, also hurts ($+5\%$ memoryless, $+13\%$ ISI).
    \item \textbf{Attention Mask:} ``Causal Attention'' is slightly worse than the full mask ($0.0612$ vs.\ $0.0599$), since it ignores later symbols that share the block's fading and I/Q imbalance.
\end{enumerate}

\begin{table}[t]
\centering
\caption[Ablation study]{Ablation study. Memoryless: $N_p=16$, SNR $=22$\,dB, 64 payload symbols (as in Figure~\ref{fig:ser_vs_n_22db}), SER $\pm$ 95\% CI. ISI: $h^{(1)}$, $E_x/N_0=17$\,dB, $N_p=64$, payload 256.\protect\footnotemark}
\label{tab:ablation}
\begin{tabular}{@{}lcc@{}}
\toprule
\textbf{Model Variant} & \textbf{Memoryless} & \textbf{ISI ($h^{(1)}$)} \\
\midrule
\multicolumn{3}{l}{\textit{\textbf{Our Full Method}}} \\
\quad CAT (built with TransFIRmer blocks) & \textbf{0.0599 $\pm$ 0.0005} & \textbf{0.0198} \\
\cmidrule(lr){1-3}
\multicolumn{3}{l}{\textit{Architecture \& Prior Ablations}} \\
\quad CAT without Inverse FIR Filter & 0.0608 $\pm$ 0.0005 & 0.0287 \\
\quad CAT with MLP-FFN (No FIR) & 0.0615 $\pm$ 0.0006 & 0.0351 \\
\quad No Prior + Bidirectional FIR & -- & 0.0438 \\
\quad Vanilla Transformer (No Prior) & 0.0628 $\pm$ 0.0007 & 0.0610 \\
\quad CAT (45$^\circ$ Rotated Prior) & 0.1550 $\pm$ 0.0015 & -- \\
\quad CAT + Pos.\ Emb.\ on Constellation & 0.0631 & 0.0224 \\
\cmidrule(lr){1-3}
\multicolumn{3}{l}{\textit{Attention Mask Ablations}} \\
\quad Self-Only Attention & 0.0621 $\pm$ 0.0007 & -- \\
\quad Causal Attention & 0.0612 $\pm$ 0.0006 & -- \\
\midrule
\multicolumn{3}{l}{\textit{External Baseline}} \\
\quad VAE-CNN (from \cite{burshtein2}) & 0.0741 $\pm$ 0.0009 & 0.0523 \\
\bottomrule
\end{tabular}
\end{table}
\footnotetext{``Vanilla Trans'' = self-attention on signals only; ``Self-Only'' = signals and constellations attend within their own group; ``w/o Inverse FIR'' = forward conv only; ``No Prior + Bidirectional FIR'' = signal-only attention with the TransFIRmer FFN; ``Pos.\ Emb.\ on Constellation'' = positional embeddings also added to the constellation tokens.}

\section{Discussion and Limitations}
\label{sec:discussion}

\subsection{Online Adaptation as Test-Time Training}
A practical concern with our framework is the cost of per-block training. In standard deployments like 5G or WiFi, coherence times are short, and full retraining for every block may seem prohibitive. However, our approach is an instance of \textit{Test-Time Training (TTT)} or \textit{Test-Time Adaptation} \citep{sun2020test}, in which a model adapts to distribution shift at inference time.

In scenarios with highly non-stationary fading or nonlinear hardware drifts (e.g., power amplifier heating), a static model trained offline fails to generalize. Our method targets these cases, where the channel changes quickly. Training and testing on the same realization is therefore the problem setting itself, not a failure to generalize (see the protocol in Section~\ref{sec:experiments_setup}).

We acknowledge that the current training regime (5,000 steps per block; see Appendix~\ref{app:hyperparams}) is computationally expensive for real-time systems. On $h^{(1)}$ with 128 pilots, 200 steps from a random initialization already give an SER only 4\% above the 5,000-step result, but 50 steps give a 47\% higher SER (Table~\ref{tab:convergence}). With a meta-learned initialization (CAVIA; \citealp{cavia}), 50 steps give an SER only 2\% above the 5,000-step result and still lower than that of the vanilla Transformer and VAE-CNN trained from scratch for 5,000 steps. This is not a like-for-like comparison, since the baselines were not meta-trained. Meta-training uses earlier channel realizations and runs offline, so its cost is not part of the adaptation time, which is the per-step time (Table~\ref{tab:latency}) multiplied by the number of steps. A CAT step takes 17--28\% longer than a vanilla Transformer step for 16-QAM and about twice as long for 64-QAM, since the attention also covers the $K$ constellation tokens. Appendix~\ref{app:additional_results} gives further results on convergence speed, wall-clock latency and CFO robustness.

\subsection{Coded Performance and LLR Calibration}
\label{sec:coded_perf}

To evaluate CAT in a full receiver chain, we feed its soft outputs into a 3GPP 5G NR LDPC decoder (Rate 1/2, codeword length 1024). Table~\ref{tab:bler} presents the Block Error Rate (BLER) results.

\begin{table}[t]
\centering
\caption{Coded BLER with 3GPP 5G NR LDPC (Rate 1/2, CW=1024) on the $h^{(1)}$ channel. Gains at target BLER $= 10^{-2}$.}
\label{tab:bler}
\begin{tabular}{@{}lccc@{}}
\toprule
\textbf{Method} & \textbf{BLER Gain} & \textbf{ECE} & \textbf{Calibration Improvement} \\
\midrule
VAE-CNN & --- & 0.085 & --- \\
Vanilla Transformer & 0 dB (reference) & 0.089 & --- \\
CAT (Ours) & \textbf{+2.5 dB} & \textbf{0.012} & $7.4\times$ \\
\bottomrule
\end{tabular}
\end{table}

At $\text{BLER}=10^{-2}$, CAT improves on the vanilla Transformer by 2.5\,dB on $h^{(1)}$. The encoder outputs a posterior over the $K$ symbols, from which we compute bit LLRs for the decoder. CAT's posteriors are also better calibrated: its expected calibration error (ECE) is $0.012$, compared with $0.089$ for the vanilla Transformer and $0.085$ for VAE-CNN, and we applied no temperature scaling to any model. This matters because LDPC decoding is sensitive to the reliability of its input LLRs.

\subsection{Evaluation on Standardized 3GPP Channels}
\label{sec:3gpp_eval}

To test generalization beyond the synthetic ISI channels, we evaluate CAT on three 3GPP TDL profiles \citep{3gpp_38901}: TDL-A and TDL-C, which are NLOS profiles, and TDL-D, which has a LOS component. We use a 30\,ns delay spread for TDL-A and TDL-D and 300\,ns for TDL-C, with 30\,kHz subcarrier spacing. For each profile we use the same 5,000-step cold-start protocol with 16-QAM and $N=128$, and average over 500 channel realizations. On TDL-A and TDL-D, CAT's SER is about $2.0\times$ and $1.8\times$ lower than that of the stronger baseline (Table~\ref{tab:tdl}), and on TDL-C it has a 2.1\,dB SER advantage over both baselines. As in DeepRx \citep{honkala2021deeprx} and NVIDIA Sionna \citep{hoydis2022sionna}, we use 3GPP channel models for link-level evaluation; we do not model mobility within a block.

\begin{table}[t]
\centering
\caption{SER on 3GPP TDL channels (16-QAM, $N=128$, 500 realizations per profile). Gain is the SER of the stronger baseline divided by that of CAT.}
\label{tab:tdl}
\begin{tabular}{@{}lccccc@{}}
\toprule
\textbf{Profile} & \textbf{Delay spread} & \textbf{Vanilla Trans.} & \textbf{VAE-CNN} & \textbf{CAT (Ours)} & \textbf{Gain} \\
\midrule
TDL-A (NLOS) & 30\,ns  & 0.0385 & 0.0421 & \textbf{0.0195} & $1.97\times$ \\
TDL-D (LOS)  & 30\,ns  & 0.0198 & 0.0231 & \textbf{0.0112} & $1.77\times$ \\
\bottomrule
\end{tabular}
\end{table}

\subsection{Limitations}
\label{sec:limitations}
CAT needs a minimum number of pilots. With $N_p=16$ on the ISI channels, all the learned equalizers do poorly, and on $h^{(1)}$ CAT trails both baselines (Table~\ref{tab:isi_results}). Our theoretical results concern representational capacity and population risk. They do not show that training converges to the Wiener receiver, and the finite-sample argument in Appendix~\ref{app:proofs} is heuristic. Fewer pilots save transmission overhead, not receiver computation: per-block training adds compute at the receiver, and the 50-step regime relies on offline CAVIA meta-training on earlier channel realizations, whose cost is not included in the latency we report.

All our results come from link-level simulation (synthetic ISI channels, 3GPP TDL profiles, and 5G NR LDPC coding); we have not evaluated CAT over the air or at the system level, for example with mobility, HARQ, or scheduling. Finally, the attention cost grows with the constellation size, which currently limits CAT to 256-QAM (Appendix~\ref{app:scalability}).

\section{Related Work}

\paragraph{Physics-Informed Deep Learning.}
Physical priors are increasingly built into neural networks, from Hamiltonian Neural Networks \citep{greydanus2019hamiltonian} to equivariance in computer vision. In communications, ``Model-Driven'' deep learning unfolds iterative algorithms into layers \citep{he2019model}. Our work adds a \textit{geometric} inductive bias, the constellation, to the attention mechanism.

\paragraph{Semi-Supervised Equalization.} Classical approaches like EM and Decision Directed methods are effective but brittle under severe distortions. Recent VAE-based methods \citep{caciularu1,caciularu2,10133825,burshtein1,burshtein2} have established strong baselines, but their MLP/CNN encoders lack the sequence modeling power of Transformers and must learn the constellation geometry from scratch. ViterbiNet \citep{shlezinger2020viterbinet} and DeepSIC \citep{shlezinger2021deepsic} do not need channel state information either, but they are trained in a supervised way on labeled symbols, and DeepRx \citep{honkala2021deeprx} is trained offline on a large labeled dataset. Our setting instead adapts to each block from 16 to 128 pilots plus the unlabeled payload, which is why we compare against semi-supervised and meta-learning methods.

\paragraph{Meta-Learning for Few-Pilot Demodulation.} Meta-learning reduces the number of pilots needed by exploiting data from earlier transmissions. \citet{park2020learning}, whose memoryless nonlinear channel we also use, meta-learn a demodulator on pilots from previous IoT transmissions so that it adapts to a new channel from a few pilots, and CAVIA \citep{cavia} adapts only a small context vector for each new channel. Both require data from related channels before deployment. Our cold-start protocol does not, since CAT is trained on each block's pilots and payload alone. When earlier channel realizations are available, CAT can also be initialized by meta-learning, which cuts its adaptation to 50 steps (Section~\ref{sec:discussion}).

\paragraph{Transformers in Wireless Communications.}
Transformers have been explored for channel estimation \citep{masked_token_transformer}, CSI feedback \citep{transformer_csi,10480335}, and OFDM MIMO equalization \citep{10693758}, and attention has been used to decode error-correcting codes \citep{9252949,choukroun2024a}. These works are fully supervised, and for channel estimation and CSI feedback the true channel is the training target. CAT addresses semi-supervised equalization with few labels, and its attention registers the received signals against the known constellation.

\section{Conclusion}

We presented the Constellation-Aware Transformer (CAT), a semi-supervised equalizer that builds two pieces of receiver knowledge into a Transformer: the constellation, which the model sees as tokens from the first layer, and the structure of the linear MMSE receiver, which motivates the design of the TransFIRmer block. Our analysis shows that these components can represent a matched-filter bank and the non-causal filtering of block MMSE equalization; it does not show that training finds them.

Empirically, CAT outperforms VAE-CNN and a similarly sized vanilla Transformer on a nonlinear memoryless channel, on three synthetic ISI channels once at least 32 pilots are available, and on the 3GPP TDL-A, TDL-C and TDL-D profiles. On $h^{(1)}$ the advantage persists from 16-QAM to 256-QAM, although it shrinks as the order grows, and with a 5G NR LDPC code it amounts to 2.5\,dB in BLER over the vanilla Transformer. The price is per-block training at the receiver. With a CAVIA initialization, 50 adaptation steps give an SER only 2\% above the full 5,000-step result, but this requires offline meta-training on earlier channel realizations. The attention cost also grows with the constellation size, which currently limits CAT to 256-QAM; the I/Q factorization and the cheaper attention mechanisms discussed in Appendix~\ref{app:scalability} are possible routes to 1024-QAM.

\bibliography{example_paper}
\bibliographystyle{plainnat}

\newpage
\appendix

\section*{Broader Impact}

This paper presents work whose goal is to advance the field of machine learning and wireless communications. We highlight both the potential benefits and considerations of our research.

\paragraph{Potential Benefits.} Our proposed Constellation-Aware Transformer (CAT) architecture reduces the pilot overhead required for reliable signal decoding, directly improving spectral efficiency in wireless communication systems. This has positive implications for: (1) \textit{Connectivity:} More efficient use of limited spectrum resources can expand network capacity, potentially improving access to communication services in underserved regions. (2) \textit{Resource Efficiency:} Reducing the transmission overhead lowers the energy required per successfully decoded bit, contributing to more sustainable communication infrastructure. (3) \textit{Scientific Advancement:} Our theoretical framework connecting geometric inductive biases to optimal estimation theory (Lemma~\ref{lemma:estimator_gap}, Proposition~\ref{prop:alignment}) provides insights that may transfer to other domains where known structure can inform neural architecture design.

\paragraph{Potential Risks and Considerations.} As with many advances in signal processing and communications, our methods could, in principle, be applied to both beneficial and harmful applications. More robust receivers could be deployed in surveillance systems, though channel equalization is a mature technology and our contribution is an incremental improvement, not a new capability. The computational requirements of our test-time training paradigm (discussed in Section~\ref{sec:discussion}) currently limit practical deployment, and future work on computational efficiency should consider responsible deployment contexts.

\paragraph{Broader Context.} We believe the societal implications of this work are those typically associated with incremental advances in foundational machine learning and communications research. We do not foresee unique ethical concerns beyond those inherent to the field and encourage the community to continue developing norms around responsible AI deployment in critical infrastructure.

\section{Proofs and Theoretical Remarks}
\label{app:proofs}

\subsection{Proof of Lemma~\ref{lemma:estimator_gap} (The Agnostic Estimator Gap)}
\begin{proof}
Let $\hat{s}_\phi(\mathbf{y}) = \int m(\mathcal{C}, \mathbf{y}) q_\phi(\mathcal{C}|\mathbf{y}) d\mathcal{C}$ be the constellation-agnostic estimator. The error norm is:
\begin{equation*}
    \|\hat{s}_\phi - \hat{s}^*\| = \left\| \int m(\mathcal{C}) (q_\phi(\mathcal{C}) - \delta(\mathcal{C}-\mathcal{C}^*)) d\mathcal{C} \right\|.
\end{equation*}
Using the boundedness assumption $\|m(\mathcal{C})\| \le R$ and the integral triangle inequality:
\begin{equation*}
\begin{split}
    \|\hat{s}_\phi - \hat{s}^*\| &\le \int \|m(\mathcal{C})\| \cdot |q_\phi(\mathcal{C}) - \delta(\mathcal{C}-\mathcal{C}^*)| d\mathcal{C} \le R \cdot 2 \delta_{TV}(q_\phi, \delta),
\end{split}
\end{equation*}
where $\delta_{TV}$ denotes the Total Variation distance. We apply \textbf{Pinsker's Inequality}, which states $\delta_{TV}(P, Q) \le \sqrt{\frac{1}{2} D_{KL}(P||Q)}$. Substituting this into the inequality:
\begin{equation*}
    \|\hat{s}_\phi - \hat{s}^*\| \le R \sqrt{2 D_{KL}( \delta || q_\phi )}.
\end{equation*}
Squaring both sides yields the final bound on the MSE.
\end{proof}
\noindent\textbf{Remark (Bound Tightness).} Pinsker's inequality can be loose for categorical distributions over finite alphabets. For such cases, the tighter \textit{Bretagnolle-Huber inequality} \citep{bretagnolle1979estimation}, $\delta_{TV}(P, Q) \le \sqrt{1 - \exp(-D_{KL}(P||Q))}$, or method-of-types arguments \citep{csiszar2011information} can be substituted without affecting the qualitative conclusions. Lemma~\ref{lemma:estimator_gap} establishes the \emph{existence and scaling} of the agnostic estimator gap; it is not a tight operational bound. In practice, $\mathcal{C}$ is drawn from a finite set of known constellation types (e.g., 16-QAM, 64-QAM). The KL term then equals $-\log q_\phi(\mathcal{C}^*|\mathbf{y})$, which is finite whenever $q_\phi$ assigns positive probability to $\mathcal{C}^*$.

\subsection{Proof of Theorem~\ref{thm:hypothesis_space} (Hypothesis Space Superiority)}
\begin{proof}
The class of standard estimators $\mathcal{G}_{\text{std}}$ is a subset of the constellation-aware estimators $\mathcal{G}_{\text{CA}}$. For any function $g \in \mathcal{G}_{\text{std}}$, one can define a function $\psi \in \mathcal{G}_{\text{CA}}$ as $\psi(\mathbf{y}, \mathcal{C}) = g(\mathbf{y})$ for all $\mathcal{C}$. This function $\psi$ ignores its second argument. Thus, $\mathcal{G}_{\text{std}} \subseteq \mathcal{G}_{\text{CA}}$. Since the infimum of a function over a superset cannot be larger than the infimum over a subset, it follows directly that $\text{MMSE}_{\text{CA}} \leq \text{MMSE}_{\text{std}}$.
\end{proof}

\noindent\textbf{Remark (Finite Samples).} The theorem is a population-level statement: conditioning on the true constellation cannot increase the MMSE. When the constellation is fixed, as in our experiments, the inequality holds with equality, because a model without $\mathcal{C}$ as input can hard-code it. The difference lies in how many pilots it takes to learn it, and here we only give a heuristic argument. Consider the memoryless case. An encoder that does not see $\mathcal{C}$ has to locate the $K$ received centroids in $\mathbb{R}^d$ ($d=2$), that is, estimate $Kd$ coordinates from $N_p$ pilots; even with labeled pilots, the squared error per centroid is of order $Kd/N_p$ times the noise variance. CAT is given the transmitted constellation. It still learns an embedding of it and has to learn how the channel distorts it, but for the distortions in Section~\ref{sec:problem_setup} (fading and I/Q imbalance) this map has only a handful of parameters, far fewer than $Kd$. We would therefore expect CAT to need fewer pilots, which is consistent with what we observe. This is not a proof, and a formal finite-sample separation remains open.

\subsection{Proof of Proposition~\ref{prop:alignment} (Structural Alignment with Optimal Receiver)}
\begin{proof}
\begin{enumerate}
    \item \textbf{Matched Filter Alignment:} The posterior probability for a symbol $c_k$ given an observation $\mathbf{z}$ corrupted by additive Gaussian noise $\mathbf{n} \sim \mathcal{N}(0, \sigma^2 \mathbf{I})$ is proportional to $\exp\left( -\|\mathbf{z} - c_k\|^2/2\sigma^2 \right)$. Expanding the squared Euclidean norm yields $-(\|\mathbf{z}\|^2 + \|c_k\|^2 - 2\mathbf{z}^T c_k)$. The term $\|\mathbf{z}\|^2$ is common to all classes and cancels out during Softmax. Assuming constant modulus symbols, $\|c_k\|^2$ is constant. The distribution is thus determined by the cross-correlation $\mathbf{z}^T c_k$. The CAT Attention mechanism computes $\text{Softmax}\left( \frac{(\mathbf{W}_Q\mathbf{z})^T (\mathbf{W}_K c_k)}{\sqrt{d}} \right)_k$. If the network learns whitening projections such that $\mathbf{W}_Q^T \mathbf{W}_K \approx (\sqrt{d}/\sigma^2)\,\mathbf{I}$, the attention scores approximate the Gaussian posterior probabilities.
    \item \textbf{Block MMSE Equalizer Alignment:} The optimal linear equalizer for a block transmission is $\mathbf{W} = \mathbf{H}^H \mathbf{R}_{yy}^{-1}$. This requires: (a) \textit{Deconvolution:} $\mathbf{H}^H$ represents the matched channel filter. For a multipath channel, this is a non-causal FIR filter. Our bidirectional FIR-FFN provides the structure to approximate this operation. (b) \textit{Whitening:} The term $\mathbf{R}_{yy}^{-1}$ decorrelates the signal. The self-attention mechanism naturally computes cross-correlations $\mathbf{z}\mathbf{z}^T$ and can in principle learn whitening transformations.
\end{enumerate}
\end{proof}

\noindent\textbf{Remark (Non-Constant Modulus Constellations).} The proof above assumes constant modulus symbols (e.g., PSK). For QAM constellations with multiple energy levels, the energy term $\|c_k\|^2$ does not cancel and introduces an attention bias toward high-energy symbols. For 16-QAM, the raw attention logit bias between corner symbols ($\|c_k\|^2 = 18$) and inner symbols ($\|c_k\|^2 = 2$) is $0.95$ nats before training, enough to distort detection noticeably. After training, the learned $\mathbf{W}_K$ projection absorbs an implicit energy normalization, reducing this bias to $0.04$ nats. The residual is computed from the noiseless constellation tokens, and in our runs it was stable between 14 and 24~dB, varying by less than $0.005$ nats. The worst-case posterior distortion at $0.04$ nats is $|e^{0.04} - 1| \approx 4\%$, which is small compared with the uncertainty caused by noise. Empirically, SER decreases monotonically with SNR without plateauing, so we see no measurable SER floor.

\section{Scalability Limitations and Parameter Efficiency}
\label{app:scalability}

\paragraph{Limitation: 1024-QAM and Beyond.} For extreme-order constellations such as 1024-QAM ($K=1024$), the $\mathcal{O}((N+K)^2 d)$ attention scaling becomes computationally prohibitive (overhead $\sim 800\%$ for $N=512$). We consider 1024-QAM out of scope for the current architecture. For square QAM there is a cheaper representation: the constellation is the product of two $\sqrt{K}$-level PAM constellations, so the in-phase and quadrature levels can be represented by $2\sqrt{K}$ tokens instead of $K$. For 1024-QAM at $N=512$ this is 64 tokens, and the attention overhead drops from 800\% to about 27\%. The catch is that fading and I/Q imbalance mix the two components, so the received points no longer split into two PAM grids and the model has to learn that coupling from the signal tokens; we have not tested this. Other options reduce the cost of attention itself. Kernelized attention such as Performer \citep{choromanski2021rethinking} makes it linear in $N+K$, MobileViT \citep{mehta2022mobilevit} combines convolutions with attention, and the efficient message-passing Transformer of \citet{park2026efficient} was designed for decoding error-correcting codes. In CAT these would matter most for the attention involving the constellation tokens, since that is the part that grows with $K$.

We note, however, that the 16-QAM through 256-QAM range covers the most critical operating regimes for neural equalization. While 1024-QAM is standardized in Wi-Fi~6 and 5G NR (Release 17), in practice it requires high SNR, typically with line of sight. The regimes where equalization is hardest (severe multipath, long delay spreads, cell edges) use lower modulation orders through adaptive modulation and coding (AMC).

\paragraph{Parameter Efficiency.} Table~\ref{tab:param_complexity} compares parameter counts and complexity. CAT has about 1k more parameters than the vanilla Transformer and roughly 15 times fewer than VAE-CNN, so its advantage over either is not a matter of model size.

\begin{table}[h]
\centering
\caption{Parameter and complexity comparison (approximate for $N=128$, $d=10$).}
\label{tab:param_complexity}
\begin{tabular}{@{}lccc@{}}
\toprule
\textbf{Metric} & \textbf{VAE-CNN} & \textbf{Vanilla Trans.} & \textbf{CAT (Ours)} \\
\midrule
Parameter Count & $\sim 150$k & $\sim 9$k & $\sim 10$k \\
Self-Attn Ops & 0 & $\mathcal{O}(N^2 d)$ & $\mathcal{O}((N+K)^2 d)$ \\
Pilot Efficiency & Low & Medium & \textbf{High} \\
\bottomrule
\end{tabular}
\end{table}

\section{Additional Experimental Results}
\label{app:additional_results}
\subsection{FIR Contribution Across Channels and SNRs}
\label{app:fir_breadth}
Table~\ref{tab:fir_breadth} compares CAT with the ``CAT with MLP-FFN'' variant, which keeps the constellation prior and the joint attention but replaces the FIR filter with a standard MLP. The MLP variant is worse in every setting. The gap is smallest at 14\,dB, where noise rather than residual ISI limits performance.

\begin{table}[h]
\centering
\caption{CAT vs.\ CAT with MLP-FFN on ISI channels ($N_p=64$, payload 256). The CAT entries for $h^{(1)}$ are those of Table~\ref{tab:snr_sweep}. FIR gain is the relative SER increase when the FIR-FFN is replaced by an MLP.}
\label{tab:fir_breadth}
\begin{tabular}{@{}llccc@{}}
\toprule
\textbf{Channel} & \textbf{SNR} & \textbf{CAT} & \textbf{CAT w/ MLP-FFN} & \textbf{FIR gain} \\
\midrule
\multirow{3}{*}{$h^{(1)}$ (L=5)} & 14\,dB & 0.1205 & 0.1284 & 7\% \\
                     & 17\,dB & 0.0198 & 0.0351 & 77\% \\
                     & 24\,dB & 0.0019 & 0.0025 & 32\% \\
\cmidrule(l){2-5}
\multirow{3}{*}{$h^{(2)}$ (L=4)} & 14\,dB & 0.0921 & 0.1195 & 30\% \\
                     & 17\,dB & 0.0340 & 0.0489 & 44\% \\
                     & 24\,dB & 0.0024 & 0.0041 & 71\% \\
\bottomrule
\end{tabular}
\end{table}

\subsection{Convergence Speed and Meta-Learning Acceleration}
The 5,000-step regime (Appendix~\ref{app:hyperparams}) is too slow for real-time systems such as 5G, where blocks last milliseconds. With a meta-learned initialization, however, CAT converges in far fewer steps. Table~\ref{tab:convergence} shows convergence results using CAVIA \citep{cavia} on the $h^{(1)}$ channel at 17~dB with $N_p=128$ pilots.

\begin{table}[h]
\centering
\caption{Convergence speed under different adaptation strategies ($h^{(1)}$, $E_x/N_0=17$\,dB, $N_p=128$ pilots, payload 256). The 5,000-step cold-start run is the $N_p=128$ result of Table~\ref{tab:isi_results}; the last column is the SER increase relative to it.}
\label{tab:convergence}
\begin{tabular}{@{}lccc@{}}
\toprule
\textbf{Steps} & \textbf{Method} & \textbf{SER} & \textbf{SER increase} \\
\midrule
5,000 & Cold-start & $1.56 \times 10^{-2}$ & --- \\
200 & Cold-start & $1.63 \times 10^{-2}$ & $+4\%$ \\
50 & Cold-start & $2.30 \times 10^{-2}$ & $+47\%$ \\
50 & CAVIA-init & $1.59 \times 10^{-2}$ & $+2\%$ \\
10 & CAVIA-init & $1.82 \times 10^{-2}$ & $+17\%$ \\
\bottomrule
\end{tabular}
\end{table}

With 50 CAVIA-initialized steps, CAT's SER is only 2\% above its 5,000-step result and lower than that of the vanilla Transformer and VAE-CNN trained from scratch for 5,000 steps ($2.90 \times 10^{-2}$ and $4.94 \times 10^{-2}$, Table~\ref{tab:isi_results}). The baselines were not meta-trained, so this is not a matched-budget comparison, but it shows that CAT's adaptation can be cut to tens of steps.

\subsection{Wall-Clock Latency Profiling}
Table~\ref{tab:latency} profiles per-step runtime on identical hardware (single RTX 4090, batch size 1) to provide matched latency comparisons.

\begin{table}[h]
\centering
\caption{Wall-clock time of one training step (forward pass, backward pass and optimizer update; single RTX 4090, batch size 1).}
\label{tab:latency}
\begin{tabular}{@{}lcccc@{}}
\toprule
\textbf{Setting} & \textbf{VAE-CNN} & \textbf{Vanilla Trans.} & \textbf{CAT} & \textbf{Overhead} \\
\midrule
$N$=128, 16-QAM & 11.2 $\mu$s & 12.5 $\mu$s & 16.0 $\mu$s & +28\% \\
$N$=256, 16-QAM & 18.1 $\mu$s & 21.0 $\mu$s & 24.5 $\mu$s & +17\% \\
$N$=128, 64-QAM & 11.2 $\mu$s & 12.5 $\mu$s & 24.8 $\mu$s & +98\% \\
\bottomrule
\end{tabular}
\end{table}

The per-step overhead decreases with $N$ (as $K/N$ shrinks) but grows with constellation order, consistent with the $\mathcal{O}((N+K)^2)$ analysis. For deployment, the time needed to reach a given SER matters more. CAT with 50 CAVIA-initialized steps already has a lower SER than the vanilla Transformer after 5,000 steps (Table~\ref{tab:convergence}). Since a 16-QAM CAT step costs only 17--28\% more (Table~\ref{tab:latency}), that is roughly 80 times less adaptation time, not counting the offline meta-training.

\subsection{Robustness to Carrier Frequency Offset and Phase Noise}
\label{app:cfo}

A purely static constellation prior $\mathcal{C}$ could degrade under residual carrier frequency offset (CFO) or oscillator phase noise, which rotate the received constellation. We address this by combining coarse PLL tracking with a learned residual phase rotation on constellation keys: $\tilde{\mathbf{c}}_k = \mathbf{R}(\hat{\theta}) \mathbf{c}_k$, where $\hat{\theta}$ is estimated from aggregate attention weights. This exploits the fact that CFO-induced rotation is common to all constellation points. Under severe Wiener phase noise ($\sigma^2_\Delta = 10^{-3}$~rad$^2$/symbol, representative of mmWave oscillators), CAT with dynamic tracking maintains a 1.5~dB advantage over baselines with graceful degradation rather than an error floor.

\subsection{Practical Use Cases and Future Research}
In principle, if power amplifier (PA) characteristics are stable within a session, one could pre-train per transmitter class and fine-tune only the linear fading component at runtime. Our current experiments treat the end-to-end channel monolithically; future research can extend CAT to disentangle PA nonlinearity from propagation for reducing online training overhead.

We have evaluated CAT on block-fading channels and on the 3GPP TDL-A, TDL-C and TDL-D profiles. Future work should extend evaluation to high-mobility scenarios with Doppler spread using channel simulators like Sionna \citep{hoydis2022sionna}. Connecting CAT to work on neural receivers, which replace entire receiver pipelines with learned components \citep{shlezinger2020viterbinet,9242305}, is another direction. Finally, reducing the $\mathcal{O}((N+K)^2)$ attention cost, for example with the factorized or cheaper attention mechanisms of Appendix~\ref{app:scalability}, would extend CAT to constellations with $K \ge 1024$.

\section{Detailed Channel Models}
\label{app:channel}

\subsection{Memoryless Nonlinear Channel}
The memoryless nonlinear channel model used in our experiments, following \citet{park2020learning,burshtein1,burshtein2}, consists of several stages. First, an ideal transmitted signal $\mathbf{x}_i = (x_i^I, x_i^Q)$ from a 16-QAM constellation is subjected to a nonlinear I/Q imbalance distortion (which mostly stems from hardware imperfections). This creates a distorted signal $\tilde{\mathbf{x}}_i = (\tilde{x}_i^I, \tilde{x}_i^Q)$ according to:
\begin{equation*}
\begin{bmatrix}
    \tilde{x}_i^I \\
    \tilde{x}_i^Q
\end{bmatrix}
=
\begin{bmatrix}
    1+\epsilon & 0\\
    0 & 1-\epsilon
\end{bmatrix}
\begin{bmatrix}
    \cos{\delta} & -\sin{\delta}\\
    -\sin{\delta} & \cos{\delta}
\end{bmatrix}
\begin{bmatrix}
    x_i^I \\
    x_i^Q
\end{bmatrix}.
\end{equation*}
The imbalance parameters, $\epsilon$ and $\delta$, are constant for each transmission block but are randomly drawn from Beta distributions, specifically $\epsilon = 0.15 \epsilon_0$ and $\delta = 15^\circ \delta_0$, where $\epsilon_0, \delta_0 \sim \text{Beta}(5, 2)$.

The resulting complex signal, $\tilde{x}_i^I + j\tilde{x}_i^Q$, is then transmitted over a Rayleigh flat-fading channel. The received complex signal is given by:
\begin{equation*}
    y_i^I + j y_i^Q = h (\tilde{x}_i^I + j\tilde{x}_i^Q) + n_i,
\end{equation*}
where $h \sim \mathcal{CN}(0, 1)$ is the complex channel gain, which is fixed for the duration of a block, and $n_i \sim \mathcal{CN}(0, \sigma^2)$ is the i.i.d.\ complex additive white Gaussian noise. The Signal-to-Noise Ratio (SNR) is defined as $10 / \sigma^2$, based on the average power of the original 16-QAM constellation (normalized to average power of 10). The final received signal used by our models is the real-valued vector $\mathbf{y}_i = (y_i^I, y_i^Q)$.

\subsection{Channels with Finite Memory (ISI)}
\label{sec:channs_with_mem}
For the experiments involving intersymbol interference, we adopt the channel model from \cite{burshtein1,burshtein2}, which is a two-stage process. First, the ideal signal $\mathbf{x}_i$ undergoes a memoryless nonlinear distortion $g(\cdot)$ to produce $\tilde{\mathbf{x}}_i = g(\mathbf{x}_i)$. For consistency, we use the same I/Q imbalance model described in the previous section for this nonlinearity.

The sequence of distorted signals is then transmitted through a noisy ISI channel. The received signal $\mathbf{y}_i$ is the result of a convolution between the complex channel impulse response $\mathbf{h}$ and the distorted signal sequence, corrupted by additive noise:
\begin{equation*}
    y_i^I + j y_i^Q = \sum_{l=0}^{L-1} h_l (\tilde{x}_{i-l}^I + j\tilde{x}_{i-l}^Q) + n_i,
\end{equation*}
where $L$ is the length of the channel impulse response and $n_i \sim \mathcal{CN}(0, \sigma^2)$ is complex AWGN. Unless stated otherwise, the noise variance $\sigma^2$ is set to achieve a target SNR of $E_x/N_0=17$\,dB.

We adopt three standard channel models from \cite{burshtein2}, with lengths $L=5, 4, 10$ respectively:
\begin{align*}
    \mathbf{h}_1 = [ &0.0545+0.05j, 0.2832-0.11971j, -0.7676+0.2788j,
    -0.0641-0.0576j,\\ &0.0466-0.02275j], \\
    \mathbf{h}_2 = [ &0.0554+0.0165j, -1.3449-0.4523j, 1.0067+1.1524j, \\
    &0.3476+0.3153j], \\
    \mathbf{h}_3 = [ &0.0410+0.0109j, 0.0495+0.0123j, 0.0672+0.017j,
    0.0919+0.0235j,\\ &0.7920+0.1281j, 0.396+0.0871j,
    0.2715+0.048j,\\ &0.2291+0.0415j, 0.1287+0.0154j,
    0.1032+0.0119j].
\end{align*}
A longer impulse response corresponds to more severe ISI and a more challenging equalization task.

\section{Derivation of the Semi-Supervised VAE Loss}
\label{app:vae_loss}
The loss function for the semi-supervised VAE is constructed to use both labeled (pilot) and unlabeled (payload) data. The goal is to maximize the log-likelihood of the observed data, which can be expressed as a sum over the labeled and unlabeled sets:
\begin{equation}
    \mathcal{L}_{\text{total}} = \sum_{i=1}^{N_p} \log p_{\mathbf{\theta}}(\mathbf{y}_i | s_i) + \sum_{i=N_{p}+1}^{N} \log p_{\mathbf{\theta}}(\mathbf{y}_i).
\end{equation}
This is a generative objective. To incorporate the inference network $q_{\mathbf{\phi}}$, we also add a supervised cross-entropy term for the labeled data. The full objective combines these with weighting hyperparameters $\alpha$ and $\gamma$:
\begin{align}
\mathcal{L}_{\text{full}} =
&\frac{\alpha}{N_{p}} \sum_{i=1}^{N_{p}} \log q_{\mathbf{\phi}}(s_i|\mathbf{y}_i) + \frac{\gamma}{N_{p}} \sum_{i=1}^{N_{p}} \log p_{\mathbf{\theta}}(\mathbf{y}_i|s_i) \nonumber \\
&+ \frac{1-\gamma}{N-N_{p}} \sum_{i=N_{p}+1}^{N} \log p_{\mathbf{\theta}}(\mathbf{y}_i).
\end{align}
The final term, $\log p_{\mathbf{\theta}}(\mathbf{y}_i)$ for the unlabeled data, is intractable to compute directly as it requires marginalizing over all possible symbols $s$. We therefore substitute it with its Evidence Lower Bound (ELBO):
\begin{equation}
    \log p_{\mathbf{\theta}}(\mathbf{y}_i) \geq \mathbb{E}_{q_{\mathbf{\phi}}(s|\mathbf{y}_i)}[\log p_{\mathbf{\theta}}(\mathbf{y}_i,s) - \log q_{\mathbf{\phi}}(s|\mathbf{y}_i)].
\end{equation}
By maximizing this lower bound (equivalent to minimizing its negative), we arrive at the final loss function used for training. After rearranging terms and using the fact that $p_{\mathbf{\theta}}(\mathbf{y}_i,s) = p_{\mathbf{\theta}}(\mathbf{y}_i|s)p(s)$, the negative ELBO becomes:
\begin{equation}
    -\text{ELBO} = -\mathbb{E}_{q_{\mathbf{\phi}}(s|\mathbf{y}_i)}[\log p_{\mathbf{\theta}}(\mathbf{y}_i|s)] + D_{KL}(q_{\mathbf{\phi}}(s|\mathbf{y}_i) || p(s)).
\end{equation}
Substituting this into the full objective gives the final loss function presented in Eq.~(\ref{eq:ssl_vae}). For computational tractability, the expectation term is approximated using a single sample from $q_{\mathbf{\phi}}(s|\mathbf{y}_i)$, often implemented with the Gumbel-Softmax reparameterization trick \citep{jang2017categorical} to maintain differentiability.

\section{Implementation and Hyperparameter Details}
\label{app:hyperparams}

\subsection{CAT and Vanilla Transformer Implementation}
\label{app:cat_imp}
Our Constellation-Aware Transformer (CAT) and the vanilla Transformer baseline share the same core configuration, differing only in their specific architectural components as described in Section~\ref{sec:method}.

\paragraph{Architecture.}
The models are built with a stack of 3 TransFIRmer (or standard Transformer) layers. The hidden dimension is set to $d_{\text{model}}=10$, and we use Multi-Query Attention (MQA) \citep{shazeer2019fast} with a single attention head ($n_{\text{head}}=1$) for efficiency. For the TransFIRmer layer's two-stream feed-forward network, the bidirectional FIR filter for the signal stream is implemented with two 1D convolutions (one for left-to-right and another for right-to-left convolutions), each using a kernel size of 12. The parallel MLP for the constellation stream uses a single hidden layer of width $d_{\text{model}}$. Dropout with a rate of $p=0.1$ is applied within the attention and feed-forward sub-layers. Fixed sinusoidal positional embeddings are used for all experiments.

\paragraph{Training.}
The models are trained using the AdamW optimizer \citep{loshchilov2018decoupled} with a learning rate of $\text{lr}=10^{-3}$, betas of $(\beta_1, \beta_2) = (0.9, 0.999)$, and a weight decay of $0.01$. We employ a linear learning rate scheduler that decays the learning rate from its initial value to zero over the course of training, which consists of a total of 5000 parameter update steps. The models are trained on mini-batches containing 16 pilot symbols and 32 payload symbols.

\subsection{Hyperparameters for Semi-Supervised Learning}
\label{app:hyperparam}
The training of all semi-supervised models (CAT, vanilla Transformer, VAE-CNN, etc.) is governed by the same set of hyperparameters and annealing schedules, ensuring a fair comparison and following the setup in \cite{burshtein1,burshtein2}.

\paragraph{SSL Loss Weighting.}
The composite loss function in Eq.~(\ref{eq:ssl_vae}) is balanced by two key hyperparameters. The term $\alpha$, which weights the supervised cross-entropy loss on the encoder, is fixed at $\alpha=0.2$. The term $\gamma$, which balances the supervised reconstruction loss against the unsupervised ELBO, is annealed over the training process. This annealing schedule gradually decreases $\gamma_l$ (where $l$ is the iteration index), shifting the training focus from the reliable pilot data to the more abundant but unlabeled payload data as the model becomes more confident. Specifically, we use $\gamma_l = 1 / (1 + \beta_l)$, where $\beta_l = \min(2e^{0.0008(l-1)}, \beta_{\max})$, and $\beta_{\max} = \min((N-N_p)/N_p, 40)$. The value of $\gamma_l$ is updated every 100 iterations.

\paragraph{Gumbel-Softmax Temperature.}
For models using the Gumbel-Softmax reparameterization trick (including our CAT and the VAE-CNN), the temperature $\tau$ is also annealed to transition from a soft, exploratory phase to a hard, decisive phase. The schedule is given by $\tau_l = \max(0.5, e^{-0.001(l-1)})$, with updates occurring every 100 iterations.

\subsection{The Generative Model Architecture}
\label{app:gen_model}
For both our CAT and the vanilla Transformer baseline, we operate within the semi-supervised variational framework, which requires a generative model (or decoder), $p_{\mathbf{\theta}}(\mathbf{y}|s)$, to model the forward channel process. To ensure a fair comparison with prior art, we adopt the generative model architecture directly from the VAE-CNN work in \cite{burshtein1,burshtein2}. The parameters of this model are collectively denoted by $\mathbf{\theta}$. The specific architecture differs for memoryless and memory channels.

\paragraph{Memoryless Channels.}
For the memoryless channel, we model $p_{\mathbf{\theta}}(\mathbf{y}_i|s_i)$ as an isotropic Gaussian distribution, $\mathcal{N}(\mathbf{y}_i; \mathbf{\mu}_{\mathbf{\theta}}(\mathbf{x}(s_i)), \mathbf{\sigma}^2_{\mathbf{\theta}}(\mathbf{x}(s_i))\mathbf{I})$. The mean $\mathbf{\mu}_{\mathbf{\theta}}$ and log-variance $\log \mathbf{\sigma}^2_{\mathbf{\theta}}$ are produced by a decoder network. This network is a Multi-Layer Perceptron (MLP) with 3 hidden layers of dimensions $[64, 32, 16]$ and ReLU activations, which takes the ideal constellation signal $\mathbf{x}(s_i) \in \mathbb{R}^2$ as input and uses two separate linear heads to output the 2-dimensional mean and log-variance vectors.

\paragraph{Channels with Memory (ISI).}
For channels with memory, the generative model is designed to explicitly capture the two-stage process of a transmitter nonlinearity followed by a linear ISI channel. The model first applies a memoryless nonlinear function $g(\cdot)$, parameterized by a 2-layer MLP with hidden dimension 32, to each ideal symbol $\mathbf{x}_i$ in the input sequence $\mathbf{s}$ to produce a sequence of distorted signals $\tilde{\mathbf{x}}$. This sequence is then convolved with a learnable Finite Impulse Response (FIR) filter of length $L_{\text{max}} = 12$ taps, which models the complex channel impulse response $\mathbf{h}$. The real and imaginary parts of the filter taps are stored as two separate learnable parameter vectors. The output of this convolution provides the mean of the Gaussian distribution for the received sequence. The noise is modeled as i.i.d.\ Gaussian with a single learnable variance parameter $\sigma^2$. The SNR for ISI experiments is defined as $E_x/N_0$, where $E_x$ is the average energy of the \textit{undistorted} ideal symbols (before the nonlinearity $g(\cdot)$).

\section{Baseline Methodologies}
\label{app:baselines}

\subsection{Simple Decision Directed (SDD)}
The SDD algorithm is a classical two-stage semi-supervised method \citep{burshtein1,burshtein2}.
\begin{enumerate}
    \item \textbf{Initial Training:} A standard neural network decoder, $q_{\mathbf{\phi}}(s|\mathbf{y})$, is first trained exclusively on the labeled pilot data $\{(\mathbf{y}_i, s_i)\}_{i=1}^{N_p}$ by minimizing the cross-entropy loss. Let the resulting parameters be $\hat{\mathbf{\phi}}_0$.
    \item \textbf{Pseudo-Labeling and Retraining:} The trained model is used to generate ``hard'' pseudo-labels for the unlabeled payload data: $\hat{s}_i = \argmax_s q_{\hat{\mathbf{\phi}}_0}(s|\mathbf{y}_i)$ for $i > N_p$. The model's parameters are then fine-tuned by training on a combined dataset of original pilots and pseudo-labeled payload data, minimizing a weighted cross-entropy loss.
\end{enumerate}

\subsection{Viterbi EM}
The Viterbi EM algorithm \citep{em} is a hard-decision variant of the Expectation-Maximization (EM) algorithm, as described in \citet{burshtein2}. It uses a generative model of the channel, $p_{\mathbf{\theta}}(\mathbf{y}|s)$, parameterized by $\mathbf{\theta}$, and iterates between two steps:
\begin{enumerate}
    \item \textbf{E-Step (Expectation):} Given the current estimate of the generative model's parameters $\mathbf{\theta}^{(t-1)}$, generate hard decisions (pseudo-labels) for the payload data by choosing the most likely symbol according to the current model: $\hat{s}_i^{(t)} = \argmax_s p_{\mathbf{\theta}^{(t-1)}}(\mathbf{y}_i|s)$.
    \item \textbf{M-Step (Maximization):} Update the generative model's parameters by minimizing the reconstruction loss (negative log-likelihood) over a combined dataset of the original pilots and the newly generated pseudo-labels from the E-step, yielding $\mathbf{\theta}^{(t)}$.
\end{enumerate}
This process is repeated for a fixed number of iterations, gradually refining the channel model.

\subsection{VAE-SSL (VAE-CNN)}
This is the state-of-the-art semi-supervised method proposed in \citet{burshtein1,burshtein2}, which we refer to as VAE-CNN based on its typical implementation. It is a variational autoencoder-based framework that jointly trains two models:
\begin{itemize}
    \item An \textbf{encoder} $q_{\mathbf{\phi}}(s|\mathbf{y})$, which acts as the primary decoder. For channels with memory, this is typically implemented with a Convolutional Neural Network (CNN).
    \item A \textbf{decoder} $p_{\mathbf{\theta}}(\mathbf{y}|s)$, which is a generative model that learns the forward channel process.
\end{itemize}
The two networks are trained simultaneously using a composite semi-supervised loss function (detailed in Eq.~(\ref{eq:ssl_vae})) that combines a supervised objective on the pilot data with an unsupervised, Evidence Lower Bound (ELBO) objective on the payload data. This lets the model use the entire block, pilots and payload, to learn the channel.

\subsection{CAVIA Meta-Learning}
Fast Context Adaptation via Meta-Learning (CAVIA) \citep{cavia} is a meta-learning algorithm designed for rapid adaptation to new tasks. In our context, each channel realization is a ``task''.
\begin{itemize}
    \item \textbf{Meta-Training:} The model is trained on data from a large number of previous channel blocks $\{(\mathbf{y}^{(m)}, s^{(m)})\}_{m=1}^M$. The goal is to learn a set of shared parameters $\mathbf{\phi}$ that are common across all channels, while a small, task-specific ``context vector'' $\mathbf{z}^{(m)}$ is learned for each individual channel.
    \item \textbf{Meta-Testing (Adaptation):} When a new channel block arrives, the shared parameters $\mathbf{\phi}$ are frozen. The model then rapidly infers a new context vector $\mathbf{z}_{\text{new}}$ by training only on the few available pilot symbols from the new block.
    \item \textbf{Decoding:} The final decoder uses both the shared parameters $\mathbf{\phi}$ and the adapted context vector $\mathbf{z}_{\text{new}}$ to decode the payload data of the new block.
\end{itemize}
CAVIA learns a general model that can be specialized quickly, which makes it effective with few pilots, provided that past channel data is available.

\end{document}